\documentclass[nohyperref]{article}

\usepackage{microtype}
\usepackage{graphicx}
\usepackage{subfigure}
\usepackage{booktabs} 
\usepackage{comment}

\usepackage{hyperref}

\newcommand{\RB}{\mathbb{R}}

\usepackage[accepted]{icml2022}

\usepackage{amsmath, physics}
\usepackage{amssymb}
\usepackage{mathtools}
\usepackage{amsthm}

\usepackage[capitalize,noabbrev]{cleveref}

\theoremstyle{plain}
\newtheorem{theorem}{Theorem}[section]

\newtheorem{lemma}[theorem]{Lemma}

\theoremstyle{definition}

\theoremstyle{plain}
\newtheorem{remark}[theorem]{Remark}

\usepackage[textsize=tiny]{todonotes}

\icmltitlerunning{On Regularizing  Coordinate-MLPs}

\begin{document}

\twocolumn[
\icmltitle{On Regularizing  Coordinate-MLPs}



\icmlsetsymbol{equal}{*}

\begin{icmlauthorlist}
\icmlauthor{Sameera Ramasinghe}{yyy}
\icmlauthor{Lachlan MacDonald}{yyy}
\icmlauthor{Simon Lucey}{yyy}
\end{icmlauthorlist}

\icmlaffiliation{yyy}{University of Adelaide}

\icmlcorrespondingauthor{Sameera Ramasinghe}{sameera.ramasinghe@adelaide.edu.au}

\icmlkeywords{COordinate-MLPs, Regularization}

\vskip 0.3in
]



\printAffiliationsAndNotice{\icmlEqualContribution} 

\begin{abstract}
We show that typical implicit regularization assumptions for deep neural networks (for regression) do not hold for coordinate-MLPs, a family of MLPs that are now ubiquitous in computer vision for representing high-frequency signals. Lack of such implicit bias disrupts smooth interpolations between training samples, and hampers generalizing across signal regions with different spectra. We investigate this behavior through a Fourier lens and uncover that as the bandwidth of a coordinate-MLP is enhanced, lower frequencies tend to get suppressed unless a suitable prior is provided explicitly. Based on these insights, we propose a simple regularization technique that can mitigate the above problem, which can be  incorporated into existing networks without any architectural modifications.
\end{abstract}

\section{Introduction}

It is well-established that deep neural networks (DNN), despite mostly being used in the over-parameterized regime, exhibit remarkable generalization properties without explicit regularization \cite{neyshabur2014search, zhang2017understanding, savarese2019infinite, goodfellow2016deep}. This behavior questions the classical theories that predict an inverse relationship between model complexity and generalization, and is often referred to in the literature as the ``implicit regularization" (or implicit bias) of DNNs \cite{li2018learning, kubo2019implicit, soudry2018implicit, poggio2018theory}. Characterizing this surprising phenomenon has been the goal of extensive research in recent years.

In contrast to this mainstream understanding, we show that coordinate-MLPs (or implicit neural networks), a class of MLPs that are specifically designed to overcome the \emph{spectral bias} of regular MLPs, do not follow the same behavior. Spectral bias, as the name suggests, refers to the propensity of DNNs to learn functions with low frequencies, making them unsuited for encoding signals with high frequency content. Coordinate-MLPs, on the other hand, are architecturally modified MLPs (via specific activation functions \cite{sitzmann2020implicit, ramasinghe2021beyond} or positional embedding schemes \cite{mildenhall2020nerf, zheng2021rethinking}) that can learn functions with high-frequency components. By virtue of this unique ability, coordinate-MLPs are now extensively being used in many computer vision tasks for representing signals including texture generation \cite{henzler2020learning,  oechsle2019texture, henzler2020learning, xiang2021neutex}, shape representation \cite{chen2019learning, deng2020nasa, tiwari2021neural, genova2020local, basher2021lightsal, mu2021sdf, park2019deepsdf}, and novel view synthesis \cite{mildenhall2020nerf, niemeyer2020differentiable, saito2019pifu, sitzmann2019scene, yu2021pixelnerf,  pumarola2021d, pumarola2021d, rebain2021derf, martin2021nerf, wang2021nerf, park2021nerfies}. However, as we will show in this paper, these architectural alterations entail an unanticipated drawback:  coordinate-MLPs, trained by conventional means via stochastic gradient descent (SGD), are incapable of simultaneously generalizing well at both lower and higher ends of the spectrum, and thus, are not automatically biased towards less complex solutions (Fig.~\ref{fig:derivatives}). Based on this observation, we question the popular understanding that the implicit bias of neural networks is more tied to SGD than the architecture  \cite{zhang2017understanding, zhang2021understanding}.

Strictly speaking, the term ``generalization" is not meaningful without context. For instance, consider a regression problem where the training points are sparsely sampled. Given more trainable parameters than the number of training points, a neural network can, in theory, learn infinitely many functions while achieving a zero train error. Therefore, the challenge is to learn a function within a space restricted by certain priors and intuitions regarding the problem at hand. The generalization then can be measured by the extent to which the learned function is close to these prior assumptions about the task. Within a regression problem, one intuitive solution that is widely accepted by the practitioners (at least from an engineering perspective) is to have a form of ``smooth" interpolation between the training points, where the low-order derivatives are bounded \cite{biship2007pattern}. In classical machine learning, in order to restrict the class of learned functions, explicit regularization techniques were used \cite{craven1978smoothing, wahba1975smoothing, kimeldorf1970correspondence}. Paradoxically, however, over-parameterized neural networks with extremely high capacity prefer to converge to such smooth solutions without any explicit regularization, despite having the ability to fit more complex functions \cite{zhang2017understanding, zhang2021understanding}.

Although this expectation of smooth interpolation is valid for both \textit{a}) general regression problems with regular MLPs and \textit{b}) high-frequency signal encoding with coordinate-MLPs, a key difference exists between their end-goals. In a general regression problem, we do not expect an MLP to perfectly fit the training data. Instead, we expect the MLP to learn a smooth curve that achieves a good trade-off between the bias and variance (generally tied to the anticipation of noisy data), which might not exactly overlap with the training points. In contrast, coordinate-MLPs are particularly expected to \emph{perfectly fit} the training data that may include both low and high fluctuations, while interpolating smoothly between samples. This difficult task requires coordinate-MLPs to preserve a rich spectrum with both low and high frequencies (often with higher spectral energy for low frequencies, as the power spectrum of natural signals such as images tends to behave as $1/f^2$ \cite{ruderman1994statistics}). We show that coordinate-MLPs naturally do \emph{not} tend to converge to such solutions, despite the existence of possible solutions within the parameter space.

It should be pointed out that it \emph{has} indeed been previously observed that  coordinate-MLPs tend to  produce noisy solutions when the bandwidth is increased excessively \cite{tancik2020fourier, ramasinghe2021beyond, sitzmann2020implicit}. However, our work complements these previous observations: \emph{First}, the existing works do not offer an explanation on why providing a coordinate-MLP the capacity to add high frequencies to the spectrum would necessarily affect lower frequencies. In contrast, we elucidate this behavior through a Fourier lens, and show that as the spectrum of the coordinate-MLPs is enhanced via hyper-parameters or depth, the lower frequencies tend to be suppressed (for a given set of weights), hampering their ability to interpolate smoothly within low-frequency regions. \emph{Second}, we interpret this behaviour from a model complexity angle, which allows us to incorporate network-depth into our analysis. \textit{Third}, we show that this effect is common to all the types of coordinate-MLPs, \textit{i.e.,} MLPs with \textit{a}) positional embeddings \cite{mildenhall2020nerf}, \textit{b}) periodic activations \cite{sitzmann2020implicit}, and \textit{c}) non-periodic activations \cite{ramasinghe2021beyond}. Finally, we propose a simple regularization term that can enforce the coordinate-MLPs to preserve both low and high frequencies in practice, enabling better generalization across the spectrum.

Our contributions are summarized below:
\begin{itemize}
    \item We show that coordinate-MLPs are not implicitly biased towards low-complexity solutions, \textit{i.e.,} typical implicit regularization assumptions do not hold for coordinate-MLPs. Previous mainstream works try to connect the  implicit bias of neural networks to  the properties of the optimization procedure (SGD), rather than the architecture \cite{zhang2017understanding, zhang2021understanding}. On the contrary, we provide counter-evidence that the implicit bias of neural networks might indeed be strongly tied to the architecture.
    
    \item We present a general result (Theorem \ref{th:md-fourier}), that can be used to obtain the Fourier transform of shallow networks with arbitrary activation functions and multi-dimensional inputs, even when the Fourier expressions are not directly integrable over the input space. Utilizing this result we derive explicit formulae to study shallow coordinate-MLPs from a Fourier perspective. 
    
    \item Using the above expressions, we  show that shallow coordinate-MLPs tend to suppress lower frequencies when the bandwidth is increased via hyper-parameters or depth, disrupting smooth interpolations. Further, we empirically demonstrate that these theoretical insights from shallow networks extrapolate well to deeper ones.
    
    \item We propose a simple regularization technique to enforce smooth interpolations between training samples while overfitting to training points with high fluctuations, preserving a rich spectrum. The proposed technique can be easily applied to existing coordinate-MLPs without any architectural modifications.
\end{itemize}

It is also important to note that our analysis stands out from previous theoretical research \cite{tancik2020fourier, ramasinghe2021beyond, zheng2021rethinking} on coordinate-MLPs due to several aspects: \textit{a}) we work on a relatively realistic setting excluding assumptions such as infinite-width networks or linear models. Although we \emph{do} consider shallow networks, we also take into account the effect of increasing depth. \textit{b}) All our expressions are derived from the first principles and do not demand the support of additional literature such as neural tangent kernel (NTK) theory. \textit{c}) we analyze different types of coordinate-MLPs within a single framework where the gathered insights are common across each type. 

\begin{figure}[!tp]
    \centering
    \includegraphics[width=1.0\columnwidth]{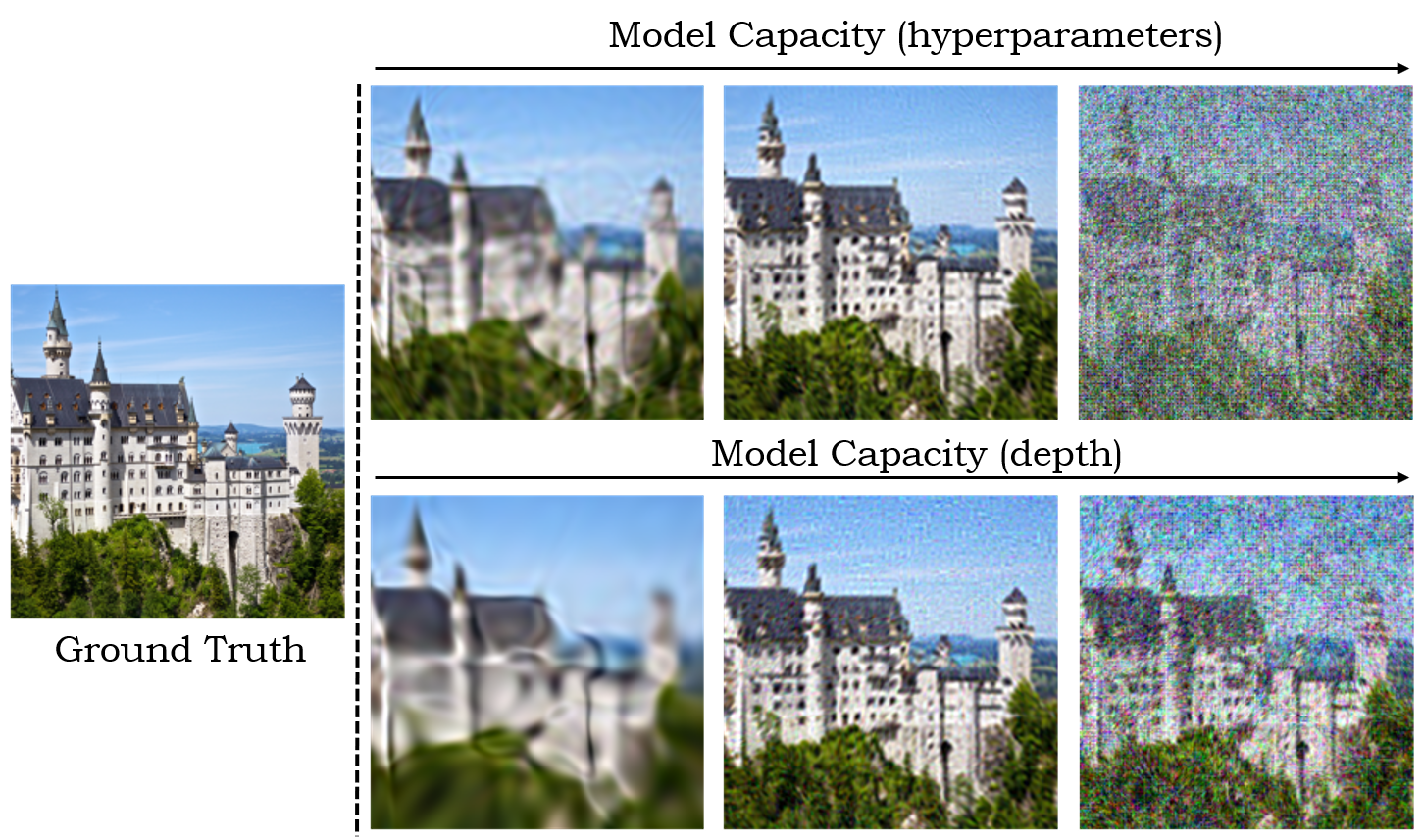}
    \vskip -0.1in
    \caption{ \textbf{Implicit regularization assumptions do not hold for Coordinate-MLPs}. The network is trained with $10\%$ of the pixels in each instance. As the capacity of the network is increased via depth or hyperparameters, coordinate-MLPs tend to produce more complex solutions even though they are trained with SGD. This is in contrast to the regular-MLPs, where the networks converge to ``smooth" solutions independent of the model capacity \cite{kubo2019implicit, heiss2019implicit}. Further, this provides evidence that the implicit regularization of neural networks might be strongly linked to the architecture rather than the optimization procedure, as opposed to mainstream understanding \cite{zhang2017understanding, zhang2021understanding}.}
    \label{fig:derivatives}
\end{figure}

\section{Related works}
\label{sec:related_works}








\textbf{Coordinate-MLPs:}  Despite the extensive empirical usage of coordinate-MLPs in various computer vision tasks, limited attention has been paid towards theoretically understanding their underlying mechanics. An interesting work by \citet{tancik2020fourier} affirmed that under the assumption of infinite-width, RFF positional embeddings allow tuning the bandwidth of the corresponding neural tangent kernel (NTK) of the MLP.  \citet{zheng2021rethinking} attempted to develop a unified framework to analyze positional embeddings from a non-Fourier lens. In particular, they showed that the performance of a positional embedding layer is governed by two factors: \textit{a}) the stable rank and \textit{ b}) the distance-preservation of the representation formed by the embedding layer.  \citet{sitzmann2020implicit}, in an alternative approach, proposed sinusoidal activation functions that enable the coordinate-MLPs to discard positional embedding layers. Proceeding one step further, \citet{ramasinghe2021beyond} focused on developing a generic framework to analyze the effectiveness of a given activation function in coordinate-MLPs. They showed that the eigenvalue distribution of the output of a hidden layer could be used as a proxy measure to the Lipschitz continuity of the network. Based on this finding, they proposed a class of novel activation functions. 

\textbf{Implicit regularization of DNNs: } Understanding and characterizing the astonishing implicit generalization properties of DNNs has been an important research topic in recent years. Despite the overwhelming empirical evidence, establishing a rigorous theoretical underpinning for this behavior has been a challenge to this date. The related research can be broadly categorized into two: investigating implicit regularization on the \textit{a}) weight space \cite{bishop1995regularization, soudry2018implicit, poggio2018theory, gidel2019implicit} and the \textit{b}) function space \cite{maennel2018gradient, kubo2019implicit, heiss2019implicit}.  Notably, \citet{maennel2018gradient} analyzed ReLU networks under macroscopic assumptions and affirmed that for given input data, there are only finitely many, ``simple" functions that can be obtained using regular ReLU MLPs, independent of the network size.   \citet{kubo2019implicit} showed that ReLU-MLPs interpolate between training samples almost linearly, and  \citet{heiss2019implicit} took a step further, proving that the interpolations can converge to (nearly) spline approximations. Seminal works by \citet{zhang2017understanding} and \cite{zhang2021understanding} showed that these generalization properties are strongly connected to the SGD. In contrast, we show that coordinate-MLPs tend to converge to complex solutions (given enough capacity), despite being trained with SGD. 

Further, all above works consider shallow networks in order to obtain precise theoretical guarantees. Similarly, we also utilize shallow architectures for a part of our analysis. However, our work significantly differs from some of the above-mentioned  theoretical work, as we do not focus on establishing rigorous theoretical bounds on generalization. Rather, we uncover, to the best of our knowledge, a critically overlooked shortcoming of coordinate-MLPs, and present plausible reasoning for this phenomenon from a Fourier perspective. Also, as a remedy to this problem, we propose a simple regularization technique that can act as a prior and preserve a better spectrum.

\section{Coordinate-MLPs}
\label{Sec:coodrdinate}
This section includes a brief exposition of coordinate-MLPs. Coordinate-MLPs aim to encode continuous signals $f:\mathbb{R}^n \to \mathbb{R}^m$, \textit{e.g.,} images, sound waves, or videos, as their weights. The inputs to the network typically are low-dimensional coordinates, \textit{e.g.}, $(x,y)$ positions, and the outputs are the sampled signal values at each coordinate \textit{e.g.}, pixel intensities. The key difference between coordinate-MLPs and regular MLPs is that the former is designed to encode signals with higher frequencies -- mitigating the spectral bias of the latter -- via specific architectural modifications. Below, we will succinctly discuss three types of coordinate-MLPs.

\textbf{Random Fourier Feature (RFF) MLPs }are compositions of a positonal embedding layer and subsequent ReLU layers. Let $\Omega$ denote the probability space $\RB^{n}$ equipped with the Gaussian measure of standard deviation $2\pi\sigma>0$.  For $D\geq1$, write elements $\vb{L}$ of $\Omega^{n}$ as matrices $[\vb{l}_{1},\dots,\vb{l}_{D}]$ composed of $D$ vectors drawn independently from $\Omega$.  Then, the random Fourier feature (RFF) positional embedding $\gamma:\Omega^{D}\times\RB^{n}\rightarrow\RB^{2D}$ is

\begin{equation}
\label{eq:rff1}
    \gamma(\vb{L},\vb{x}):=[\vb{e}(\vb{l}_{1},\vb{x}),\dots,\vb{e}(\vb{l}_{D},\vb{x})]^{T},
\end{equation}

where $\vb{e}:\Omega\times\RB^{n}\rightarrow\RB^{2}$ is the random function defined by the formula

\begin{equation}
\label{eq:rff2}
    \vb{e}(\vb{l},\vb{x}):=[\sin(\vb{l}\cdot\vb{x}),\cos(\vb{l}\cdot\vb{x})].
\end{equation}



The above layer is then followed by a stack of ReLU layers.

\textbf{Sinusoidal MLPs } were originally proposed by \citet{sitzmann2020implicit}. Let $\vb{W}$ and $\vb{b}$ be the weights and the bias of a hidden layer of a sinusoidal coordinate-MLP, respectively.  Then, the output of the hidden-layer is $\mathrm{sin}(2 \pi a (\vb{W}\cdot \vb{x} + \vb{b}))$, where $a$ is a hyper-parameter that can control the spectral-bias of the network and $\vb{x}$ is the input. 


\textbf{Gaussian MLPs }are a recently proposed type of coordinate-MLP by \citet{ramasinghe2021beyond}. The output of a Gaussian hidden-layer is defined as $e^{-(\frac{(\vb{W}\cdot \vb{x} + \vb{b})^2}{2\sigma^2})}$, where $\sigma$ is a hyper-parameter. 

In the next section, we will derive a general result that can be used to obtain the Fourier transform of an arbitrary multi-dimensional shallow MLP (given that the 1D Fourier transform the activation function exists), which is used as the bedrock in a bulk of our derivations later.




\section{Fourier transform of a shallow MLP.}
\label{sec:FT}




Let $\mathcal{G}: \mathbb{R}^n \to \mathbb{R}$ be an MLP with a single hidden layer with $m$ neurons, and a point-wise activation function $\alpha:\mathbb{R} \to \mathbb{R}$. Suppose we are interested in obtaining the Fourier transform $\hat{\mathcal{G}}$ of $\mathcal{G}$. Since the bias only contributes to the DC component, we formulate $\mathcal{G}$ as follows:

\begin{equation}
    \mathcal{G} = \sum_{i= 1}^{m} w^{(2)}_i\alpha(\vb{w}^{(1)}_i\cdot \vb{x}),
\end{equation}

where $\vb{w}_i^{(1)}$ are rows of the fist-layer affine weight matrix \textit{i.e.}, input to the $i^{th}$ hidden neuron is $\vb{w}^{(1)}_i\cdot \vb{x}$, and $w^{(2)}_i$ are the weights of the last layer. Since Fourier transform is a linear operation, it is straightforward to see that

\begin{equation}
    \hat{\mathcal{G}} = \sum_{i= 1}^{m} w^{(2)}_i\hat{\alpha}(\vb{w}^{(1)}_i\cdot \vb{x}).
\end{equation}

Thus, by obtaining the Fourier transform of

\begin{equation}
\label{eq:activation_fourier}
    f(\vb{x}):=\alpha(\vb{w\cdot x}),
\end{equation}

it should be possible to derive the Fourier transform of the MLP. However, doing so using the standard definition of the multi-dimensional Fourier transform can be infeasible in some cases. For example, consider a Gaussian-MLP. Then, one can hope to calculate the Fourier transform of Eq.~\ref{eq:activation_fourier} as

\begin{equation}
\label{eq:standard_gau}
    \hat{f}(\vb{k}) = \int_{\mathbb{R}^n} e^{-(\frac{(\vb{w}\cdot \vb{x})^2}{2\sigma^2} + 2\pi i \vb{x} \cdot \vb{k})} d\vb{x}.
\end{equation}

However, note that there exists an $n-1$ dimensional subspace where $(\vb{w}\cdot \vb{x})$ is zero and thus, the first term inside the exponential becomes $0$ in these cases. Therefore, Eq.~\ref{eq:standard_gau} is not integrable.
Nonetheless, $f$ defines a tempered distribution, and its Fourier transform $\hat{f}$ can therefore be calculated by fixing a Schwartz test function $\varphi$ on $\mathbb{R}^{n}$ and using the identity
\begin{equation}\label{identity}
\langle\hat{f},\varphi\rangle = \langle f,\hat{\varphi}\rangle,
\end{equation}
defining the Fourier transform of the tempered distribution $f$.  Here $\langle\cdot,\cdot\rangle$ denotes the pairing between distributions and test functions (see Chapter 8 of \cite{distributions} for this background). 




First, we present the following lemma.

\begin{lemma}\label{lemma1}
Let $A:\mathbb{R}^{n-1}\rightarrow\mathbb{R}^{n}$ be an isometric linear map, and let $A^{\perp}\subset\mathbb{R}^{n}$ denote the 1-dimensional subspace which is orthogonal to the range of $A$.  Then for any integrable function $\varphi\in L^{1}(\mathbb{R}^{n})$, one has
\[
\int_{\mathbb{R}^{n-1}}\hat{\varphi}(A\vb{x})d\vb{x} = (2\pi)^{\frac{n-1}{2}}\int_{A^{\perp}}\varphi(\vb{y})\,d_{A^{\perp}}\vb{y},
\]
where $d_{A^{\perp}}\vb{y}$ is the volume element induced on the subspace $A^{\perp}$.
\end{lemma}

\begin{proof}
First using the dominated convergence theorem, and then applying Fubini's theorem, we have
\begin{align*}
    \int\hat{\varphi}(A\vb{x})\,d\vb{x} =& \lim_{\epsilon\rightarrow0}\int e^{-\frac{\epsilon|\vb{x}|^{2}}{2}}\hat{\varphi}(A\vb{x})\,d\vb{x}\\ =& \lim_{\epsilon\rightarrow0}\int\bigg(\int e^{-i(A\vb{x})\cdot\vb{k}}e^{-\frac{\epsilon|\vb{x}|^{2}}{2}}\,d\vb{x}\bigg)\varphi(\vb{k})\,d\vb{k}\\ =& (2\pi)^{\frac{n-1}{2}}\lim_{\epsilon\rightarrow0}\int \frac{e^{-\frac{|A^{T}\vb{k}|^{2}}{2\epsilon}}}{\epsilon^{\frac{n-1}{2}}}\varphi(\vb{k})\,d\vb{k},
\end{align*}
where on the last line we have used the fact that the Fourier transform of a Gaussian is again a Gaussian.  Now, for any $\vb{y}\in A^{\perp}$, let $A^{\parallel}(\vb{y})$ denote the $(n-1)$-dimensional hyperplane parallel to the range of $A$ that passes through $\vb{y}$, and denote by $d_{A^{\parallel}(\vb{y})}\vb{z}$ the induced volume element along $A^{\parallel}(\vb{y})$.  Each $\vb{k}\in\mathbb{R}^{n}$ decomposes uniquely as $\vb{k} = \vb{y}+\vb{z}$ for some $\vb{y}\in A^{\perp}$ and $\vb{z}\in A^{\parallel}(\vb{y})$.  The integral over $\mathbb{R}^{n}$ of any integrable function $\psi$ on $\mathbb{R}^{n}$ can then be written
\[
\int \psi(\vb{k})\,d\vb{k} = \int_{A^{\perp}}\int_{A^{\parallel}(\vb{y})}\psi(\vb{y+z})d_{A^{\parallel}(\vb{y})}\vb{z}\,d_{A^{\perp}}\vb{y}
\]
Since the function $e^{-\frac{|A^{T}\vb{k}|^{2}}{2\epsilon}}$ is constant along any line parallel to $A^{\perp}$, the dominated convergence theorem now tells us that $(2\pi)^{\frac{1-n}{2}}\int\hat{\varphi}(A\vb{x})\,d\vb{x}$ is equal to
\begin{align*}
&\int_{A^{\perp}}\lim_{\epsilon\rightarrow0}\int_{A^{\parallel}(\vb{y})}\frac{e^{-\frac{|A^{T}(\vb{z})|^{2}}{2\epsilon}}}{\epsilon^{\frac{n-1}{2}}}\varphi(\vb{y+z})\,d_{A^{\parallel}(\vb{y})}\vb{z}\,d_{A^{\perp}}\vb{y}\\
=&\int_{A^{\perp}}\lim_{\epsilon\rightarrow0}\int_{\mathbb{R}^{n-1}}\frac{e^{-\frac{|\vb{x}|^{2}}{2\epsilon}}}{\epsilon^{-\frac{n-1}{2}}}\varphi(\vb{y}+A\vb{x})\,d\vb{x}\,d_{A^{\perp}}\vb{y}\\
=&\int_{A^{\perp}}\lim_{\epsilon\rightarrow0}(\eta_{\epsilon}*\varphi_{\vb{y}})(0)\,d_{A^{\perp}}\vb{y}\\=&\int_{A^{\perp}}\varphi(\vb{y})d_{A^{\perp}}\vb{y},
\end{align*}
where on the third line we have used the definitions $\varphi_{\vb{y}}(\vb{x}):=\varphi(\vb{y}+A\vb{x})$ and $\eta_{\epsilon}(\vb{x}):=\epsilon^{-\frac{n-1}{2}}e^{-\frac{|\vb{x}|^{2}}{2\epsilon}}$ of functions on $\mathbb{R}^{n-1}$, and on the final line we have used the fact that convolution with $\eta_{\epsilon}$ is an approximate identity for $L^{1}(\mathbb{R}^{n-1})$.
\end{proof}

Now, we present our main theorem.

\begin{theorem}\label{th:md-fourier}
The Fourier transform of $f$ is the distribution
\[
\hat{f}(\vb{k}) = \frac{(2\pi)^{\frac{n}{2}}}{|\vb{w}|}\hat{\alpha}\bigg(\vb{\frac{w}{|w|^{2}}\cdot k}\bigg)\delta_{\vb{w}}(\vb{k}),
\]
where $\delta_{\vb{w}}(\vb{k})$ is the Dirac delta distribution which concentrates along the line spanned by $\vb{w}$.
\end{theorem}

\begin{proof}
Let $A = \bigg[\frac{\vb{w}}{|\vb{w}|},\vb{a_{2}},\dots,\vb{a_{n}}\bigg]$ be any special orthogonal matrix for which $A\vb{e_{1}} = \frac{\vb{w}}{|\vb{w}|}$, letting $\bar{A} = [\vb{a_{2}},\dots,\vb{a_{n}}]$ denote the corresponding submatrix. Given a vector $\vb{k} = (k_{1},\dots,k_{n})^{T}\in\mathbb{R}^{n}$, use the notation $\vb{k}' = (k_{2},\dots,k_{n})^{T}\in\mathbb{R}^{n-1}$.  Then the right hand side $\langle f,\hat{\varphi}\rangle$ of Equation \eqref{identity} is given by
\begin{align*}
    &(2\pi)^{-\frac{n}{2}}\int\int e^{-i\vb{x\cdot k}}\alpha(\vb{w\cdot x})\varphi(\vb{k})\,d\vb{x}\,d\vb{k}\\
    =&(2\pi)^{-\frac{n}{2}}\int\int e^{-i \vb{y}\cdot(A^{T}\vb{k})}\alpha(\vb{w}\cdot(A\vb{y}))\varphi(\vb{k})\,d\vb{y}\,d\vb{k}\\
    =&(2\pi)^{\frac{1-n}{2}}\int\int\bigg((2\pi)^{-\frac{1}{2}}\int e^{-iy_{1}\vb{\frac{w}{|w|}\cdot k}}\alpha(|\vb{w}|y_{1})\,dy_{1}\bigg)\\
    &\times e^{-i\vb{y}'\cdot(A^{T}\vb{k})'}\varphi(\vb{k})\,d\vb{y}'\,d\vb{k}\\
    =&(2\pi)^{\frac{1-n}{2}}\int\int\frac{1}{|\vb{w}|}\hat{\alpha}\bigg(\vb{\frac{w}{|w|^{2}}\cdot k}\bigg)\varphi(\vb{k})e^{-i\vb{y}'\cdot(A^{T}\vb{k})'}\,d\vb{k}\,d\vb{y'}\\
    =&(2\pi)^{\frac{1-n}{2}}\int\int\frac{1}{|\vb{w}|}\hat{\alpha}\bigg(\vb{\frac{w}{|w|^{2}}\cdot k}\bigg)\varphi(\vb{k})e^{-i\vb{y}'\cdot(\bar{A}^{T}\vb{k})}\,d\vb{k}\,d\vb{y}'\\
    =&(2\pi)^{\frac{1-n}{2}}\int\int\frac{1}{|\vb{w}|}\hat{\alpha}\bigg(\vb{\frac{w}{|w|^{2}}\cdot k}\bigg)\varphi(\vb{k})e^{-i(\bar{A}\vb{y}')\cdot\vb{k}}\,d\vb{k}\,d\vb{y'}\\
    =&(2\pi)^{\frac{1}{2}}\int \hat{g}(\bar{A}\vb{y}')\,d\vb{y}',
\end{align*}
where for the first equality we have made the substitution $\vb{x} = A\vb{y}$, and in the final step we have set $g(\vb{k}):=\frac{1}{|\vb{w}|}\hat{\alpha}\bigg(\vb{\frac{w}{|w|^{2}}\cdot k}\bigg)\varphi(\vb{k})$.  Our integral is thus of the form given in Lemma \ref{lemma1}, and the result follows.
\end{proof}

Next, using the above result, we will gather useful insights into different types of coordinate-MLPs from a Fourier perspective.

\section{Effect of hyperparameters}
\label{sec:hyper}
In this section, we will primarily explore the effect of hyperparameters on the spectrum of coordinate-MLPs, among other key insights.

\subsection{Gaussian-MLP}
The Fourier transform of the Gaussian activation is $\hat{\alpha}(k) = \sqrt{2\pi} \sigma e^{-(\sqrt{2}\pi k\sigma)^2}$. Now, using Theorem \ref{th:md-fourier}, we can obtain the Fourier transform of a Gaussian-MLP with a single hidden layer as

\begin{equation}
\label{eq:gaussian}
   \sum_{i= 1}^{m} w^{(2)}_i \frac{(2\pi)^{\frac{n+1}{2}}\sigma}{|\vb{w}^{(1)}_i|}e^{-\Big(\sqrt{2}\pi  \frac{\vb{w}^{(1)}_i}{|\vb{w}^{(1)}_i|^{2}}\cdot \vb{k} \sigma \Big)^2} \delta_{\vb{w}^{(1)}_i}(\vb{k}).
\end{equation}

\textbf{Discussion: } For fixed $\vb{w}^{(1)}_i$'s and $\sigma$, the spectral energy is decayed as  $\vb{k}$ is increased. A practitioner can increase the energy of higher frequencies by decreasing $\sigma$. Second, although decreasing $\sigma$ can increase the energy of higher frequencies (for given weights), it simultaneously suppresses the energies of lower frequencies due to the $\sigma$ term outside the exponential. Third, one can achieve the best of both worlds (\textit{i.e.,} incorporate higher-frequencies to the spectrum while maintaining higher energies for lower-frequencies) by appropriately tuning $\vb{w}^{(1)}_i$'s. To make the above theoretical observations clearer, we demonstrate a toy example in Fig.~\ref{fig:toyex}.

\subsection{Sinusoidal-MLPs}

 Similar to the Gaussian-MLP, the Fourier transform of a sinusoidal-MLP with a single hidden-layer can be obtained as

\begin{align*}
\label{eq:sinmlp}
    &\sum_{i= 1}^{m} w^{(2)}_i \frac{(2\pi)^{\frac{n}{2}}}{2|\vb{w}^{(1)}_i|}\delta_{\vb{w}^{(1)}_{i}}(\vb{k})\times\\&\times\Bigg( \delta\bigg(\frac{\vb{w}^{(1)}_i}{|\vb{w}^{(1)}_i|^{2}}\cdot \vb{k}-a\bigg) + \delta\bigg(\frac{\vb{w}^{(1)}_i}{|\vb{w}^{(1)}_i|^{2}}\cdot \vb{k}+a\bigg)\Bigg) ,
\end{align*}

\textbf{Discussion: } According to Eq.~\ref{eq:sinmlp}, the only frequencies present in the spectrum are multiples $\vb{k} = a\vb{w}_{i}^{(1)}$ of the weight vectors $\vb{w}^{(1)}_{i}$. It follows that the magnitudes of the frequencies present in the spectrum are lower-bounded by the number $a\,\mathrm{min}\{|\vb{w}_i|\}$. Thus for a given set of weights, one can add higher frequencies to the spectrum by increasing $a$. At the same time, the spectrum can be \emph{balanced} by minimizing some $\vb{w}_i$'s accordingly.


\subsection{RFF-MLPs}
We consider a shallow RFF-MLP $\mathbb{R}^n \to \mathbb{R}$ with a single hidden ReLU layer. Recall that the positional embedding scheme is defined follwing Eq.~\ref{eq:rff1} and \ref{eq:rff2}.
The positional embedding layer $\gamma$ is then followed by a linear transformation $A:\RB^{2D}\rightarrow\RB^{2D}$ yielding
\[
f_{2}(\vb{L},\vb{x})^{i} = \sum_{j}a^{i}_{j} \gamma(\vb{L},\vb{x})^{j},
\]
which is then followed by ReLU activations and an affine transformation. ReLU activations can be approximated via linear combinations of various polynomial basis functions \cite{ali2020polynomial, telgarsky2017neural}. One such set of polynomials is Newman polynomials defined as
\[
N_{m}(x):=\prod_{i=1}^{m-1}(x+\exp(-i/\sqrt{n})).
\]
It is important to note that although we stick to the Newman polynomials in the subsequent derivations, the obtained insights are valid for \emph{any} polynomial approximation, as they only depend on the power terms of the polynomial approximation.



For $m\geq 5$, yielding
\begin{align*}
f_{3}(\vb{L},\vb{x})^{i} =& \prod_{\alpha=1}^{m-1}\bigg(\sum_{j}a^{i}_{j}f_{1}(\vb{L},\vb{x})^{j}+\exp(-\alpha/\sqrt{n})\bigg)\\ =& \sum_{\beta=0}^{m-2} \kappa(m,\beta) \big(\sum_{j}a^{i}_{j}f_{1}(\vb{L},\vb{x})^{j}\big)^{\beta}\\ =& \sum_{\beta=0}^{m-2}\kappa(m,\beta)\times\\&\sum_{k_{1}+\cdots+k_{2D}=\beta}{\beta\choose k_{1},\dots,k_{2D}}\prod_{t=1}^{2D}(a^{i}_{t}f_{1}(\vb{L},\vb{x})^{t})^{k_{t}}
\end{align*}
where $\kappa(m,\beta) = {m-2\choose \beta}\prod_{\alpha=1}^{m-1-\beta}\exp(-\alpha/\sqrt{m})$.  It follows that the spectrum of $f_{3}$ is determined by the spectra of the powers $\big(\gamma(\vb{L},\vb{x})^{j}\big)^{k}$ of the components of $f_{1}$.  Write $j = 2j_{1}+j_{2}$, where $j_{1}$ is the floor of $j/2$. Then
\[
\gamma(\vb{L},\vb{x})^{j} = \frac{e^{i\vb{l}_{j_{1}}\cdot\vb{x}}+(-1)^{j}e^{i\vb{l}_{j_{1}}\cdot\vb{x}}}{2}.
\]
Therefore, for $k\geq 0$,
\[
(\gamma(\vb{L},\vb{x})^{j})^{k} =  2^{-k}\sum_{a=0}^{k}(-1)^{j(k-a)}e^{i(2a-k)\vb{l}_{j_{1}}\cdot\vb{x}}.
\]
We then see that
\begin{align*}
\prod_{t=1}^{2D}(f{1}(\vb{l},\vb{x})^{t})^{k_{t}} =& \prod_{t=1}^{2D}2^{-k_{t}}\bigg(\sum_{a=0}^{k_{t}}(-1)^{t(k-a)}e^{i(2a-k)\vb{l}_{j_{1}}\cdot\vb{x}}\bigg)\\ =& C\sum_{a_{1},\dots,a_{2D}=0}^{k_{1},\dots,k_{2D}}\prod_{t=1}^{2D}(-1)^{t(k_{t}-a_{t})}e^{i(2a_{t}-k_{t})\vb{l}_{t_{1}}\cdot\vb{x}}\\ =& C\sum_{a_{1},\dots,a_{2D}=0}^{k_{1},\dots,k_{2D}}C'e^{i\sum_{t=1}^{2D}(2a_{t}-k_{t})\vb{l}_{t_{1}}\cdot\vb{x}},
\end{align*}
where $C = \prod_{t=1}^{2D}2^{-k_{t}}$ and $C' = \prod_{t=1}^{2D}(-1)^{t(k_{t}-a_{t})}$.  The proof of Theorem \ref{thm} can then be used to show that the Fourier transform of $f_{3}$ is a linear combination of terms of the form
\[
\vb{k}\mapsto\delta\bigg(\frac{\vb{s}\cdot\vb{k}}{|\vb{s}|^{2}}-1\bigg)\delta_{\vb{s}}(\vb{k}),
\]
where $\vb{s}= \sum_{j=1}^{D}\big((2a_{2j-1}-k_{2j-1}) + (2a_{2j}-k_{2j})\big)\vb{l}_{j}$ for positive integers $k_{1},\dots,k_{2D}$ such that $k_{1}+\cdots+k_{2D}\leq m-2$, and integers $0\leq a_{j}\leq k_{j}$ for $j=1,\dots,2D$.  It follows that the spectrum of $f_{3}$ is concentrated on frequencies
\[
\vb{k} = \sum_{j=1}^{D}\big((2a_{2j-1}-k_{2j-1}) + (2a_{2j}-k_{2j})\big)\vb{l}_{j},
\]

for non-negative integers $k_{1},\dots,k_{2D}$ adding to $m-2$, and for all non-negative integers $a_{j}\leq k_{j}$ for all $j=1,\dots,2D$.  Here we recall that $D$ is the dimension determined by the choice of RFF positional embedding.

\textbf{Discussion:} Recall that if $U, V \sim \mathcal{N}(0, \sigma^2)$, then, $Var(t_1U + t_2V) = (t_1^2+ t_2^2)\sigma^2$. Therefore the overall spectrum can be made wider by increasing the standard deviation $\sigma$ of the distribution from which the $\vb{l}_{j}$ are drawn. However, since the $D$ is constant, a larger $\sigma$ makes frequencies less concentrated in the lower end, reducing the overall energy of the low frequency components. Nonetheless, the spectrum can be altered via adjusting the network weights.

Thus far, we established that in all three types of coordinate-MLPs, a rich spectrum consisting of both low and high frequencies can be preserved by properly tuning the weights and hyperparameters. However, coordinate-MLPs tend to suppress lower frequencies when trained without explicit regularization (see Fig.~\ref{fig:depthvsspectrum}). Moreover, note that in each Fourier expression, the Fourier components cease to exist if $\vb{k}$ is not in the direction of $\vb{w}^{(1)}_i$, due to the Dirac-delta term. Therefore, the angles between  $\vb{w}^{(1)}_i$'s should be increased, in order to add frequency components in different directions. 


\section{Effect of depth}

\label{sec:depth}
This section investigates the effects of increasing depth on the spectrum of coordinate-MLPs. Comprehensive derivations of the expressions in this section can be found in the Appendix. Let a stack of layers be denoted as $\kappa:\mathbb{R}^n \to \mathbb{R}^d$. Suppose we add another stack of layers $\eta:\mathbb{R}^d \to \mathbb{R}$ on top of $\kappa(\cdot)$ to construct the MLP $\eta \circ \kappa:\mathbb{R}^n \to \mathbb{R}$.  One can rewrite $(\eta \circ \kappa)$ in terms of the inverse Fourier transform as,

\[
    (\eta \circ \kappa) = (\frac{1}{\sqrt{2 \pi}})^{d} \int_{\mathbb{R}^d} \eta(\vb{t})e^{i\vb{t \cdot \kappa(x)} }d\vb{t}
\]

Fourier transform of $(\eta \circ \kappa)$ now can be written as,

\begin{align*}
   \widehat{(\eta \circ \kappa)} & =  \Big(\frac{1}{\sqrt{2 \pi}})^{d+n} \int_{\mathbb{R}^n} \int_{\mathbb{R}^d} \eta(\vb{t})e^{i\vb{t \cdot \kappa(x)} }d\vb{t}e^{-i\vb{k\cdot x}}d\vb{x} \\
   & = \Big(\frac{1}{\sqrt{2 \pi}})^{d+n} \int_{\mathbb{R}^d} \eta(\vb{t})\int{\mathbb{R}^n} e^{i\vb{l \cdot \kappa(x)}} e^{-\vb{k \cdot x}}d\vb{x} d\vb{t},
\end{align*}

which yields,

\[
    \widehat{(\eta \circ \kappa)}(\vb{k}) = \Big( \frac{1}{\sqrt{2 \pi}} \Big)^n \langle \hat{\eta} (\cdot), \beta(\vb{k}, \cdot)  \rangle,
\]



where, $\beta(\vb{k}, \vb{t}) = \int_{\mathbb{R}^n} e^{i(\vb{t \cdot \kappa (x) - k\cdot x})}
$. That is, the composite Fourier transform is the projection of $\hat{\eta}$ on to $\beta(\cdot)$. The maximum magnitude frequency $\vb{k^*_{\eta \circ \kappa}}$ present in the spectrum of $\widehat{(\eta \circ \kappa)}$ is 

\begin{equation}
    \vb{k^*_{\eta \circ \kappa}} = \max_{|\vb{u}|=1}(\vb{k}_{\hat{\eta}, \vb{u}}^* \max_{\vb{x}}(\vb{u \cdot \kappa'(x)})),
\end{equation}

where $\vb{k}_{\hat{\eta}, \vb{u}}$ is the maximum frequency of $\hat{\eta}$ along $\vb{u}$ \cite{bergner2006spectral}. Thus, the addition of a set of layers $\eta$ with sufficiently rich spectrum  tends to increase the maximum magnitude frequency expressible by the network. Next, we focus on the impact on lower frequencies by such stacking. It is important to note that the remainder of the discussion in this section is not a rigorous theoretical derivation, but rather, an intuitive explanation. 

We consider the 1D case for simplicity. Our intention is to explore the effect of adding more layers on $\langle \hat{\eta} (\cdot), \beta(\vb{k}, \cdot)  \rangle$. Note that $\beta$ is an integral of an oscillating function with unit magnitude. The integral of an oscillating function is non-negligible only at points where the phase is close to zero, \textit{i.e.,} points of stationary phase. At the stationary phase of $\beta$, it can be seen that $\frac{du}{dx} = 0$, where $u(x) = t\kappa(x) - kx = 0$. Converting to polar coordinates,

   \begin{align*}
       \frac{du}{dx} & = \frac{d}{dx} r(\kappa(x)\mathrm{sin}\theta - x \mathrm{cos}\theta) =0 \\
       & \kappa'(x_s)\mathrm{sin} \theta - \mathrm{cos} \theta = 0\\
       & \frac{1}{\kappa'(x_s)} = \mathrm{tan}\theta,
   \end{align*}

where $x_s$ are the points at stationary phase. By using the Taylor approximation around $x_s$, we get,

\[
  I_{x_s} \sim \int_{\infty}^{\infty}e^{r(\kappa(x_s)\mathrm{sin}\theta - x_s \mathrm{cos}\theta + \frac{1}{2}\kappa''(x_s)x^2 \mathrm{sin} \theta)}dx
\]
\[
I_{x_s} \sim \int_{\infty}^{\infty}e^{r(\kappa(x_s)\mathrm{sin}\theta - x_s \mathrm{cos}\theta} \Big( \frac{2 \pi}{r|\kappa''(x_s)\mathrm{sin}\theta|} \Big)^{\frac{1}{2}}e^{i \frac{\pi}{4}sgn\{ \kappa''(x_s)\mathrm{sin}\theta \}}
\]

At points where $\kappa''(x_s)$ differs from $0$ largely, the integral vanishes as $(x - x_s)^2$ increases. Since we can obtain the full integral by summing $I_{x_s}$ for all $x_s$ such that $\frac{1}{\kappa'(x_s)} = \mathrm{tan}\theta$. Hence, around the stationary phase,

\[
  \mathrm{min}(\kappa') < \frac{1}{\mathrm{tan}\theta} < \mathrm{max}(\kappa')
\]

which is essentially,

\begin{equation}
    \mathrm{min}(\kappa'(x)) < \frac{k}{t} < \mathrm{max}(\kappa'(x)).
\end{equation}


From the previous analysis, we established that progressively adding layers increases the maximum frequency of the composite function. This causes $\kappa$ to have rapid fluctuations, encouraging $\mathrm{min}(\kappa'(x))$ to increase as the network gets deeper (note that the definition of $\kappa$ keeps changing as more layers are added). Thus, according to Eq.~\ref{eq:deep_condition}, smaller $k$'s causes the quantity $\langle \hat{\eta} (\cdot), \beta(\vb{k}, \cdot)  \rangle$ to be smaller, which encourages suppressing the energy of the lower frequencies of $\widehat{(\eta \circ \kappa)}$. Rigorously speaking, increasing the bandwidth of the network by adding more layers does not necessarily have to increase $\mathrm{min}(\kappa'(x))$. However, our experimental results strongly suggest that this is the case. Summarizing our insights from Sec.~\ref{sec:FT}, \ref{sec:hyper} and Sec.~\ref{sec:depth}, we state the following remarks.

\begin{remark}
    In order to gain a richer spectrum, the
columns of the affine weight matrices of the coordinate-
MLPs need to point in different directions. Ideally, the
column vectors should lie along the frequency directions in
the target signal.
\end{remark}

\begin{remark}
\label{rm:1}
    The spectrum of the coordinate-MLPs can be altered either by tuning the  hyperparameters or changing the depth. In practice, when higher frequencies are added to the spectrum, the energy of lower frequencies tend to be suppressed, disrupting smooth interpolations between samples. However, better solutions exists in the parameter space, thus, an explicit prior is needed to bias the network weights towards those solutions.
\end{remark}

In the next section, we will focus on proposing a regularization mechanism that can aid the coordinate-MLPs in finding a solution with a better spectrum.







\begin{figure*}
    \centering
    \includegraphics [width = 2.\columnwidth]{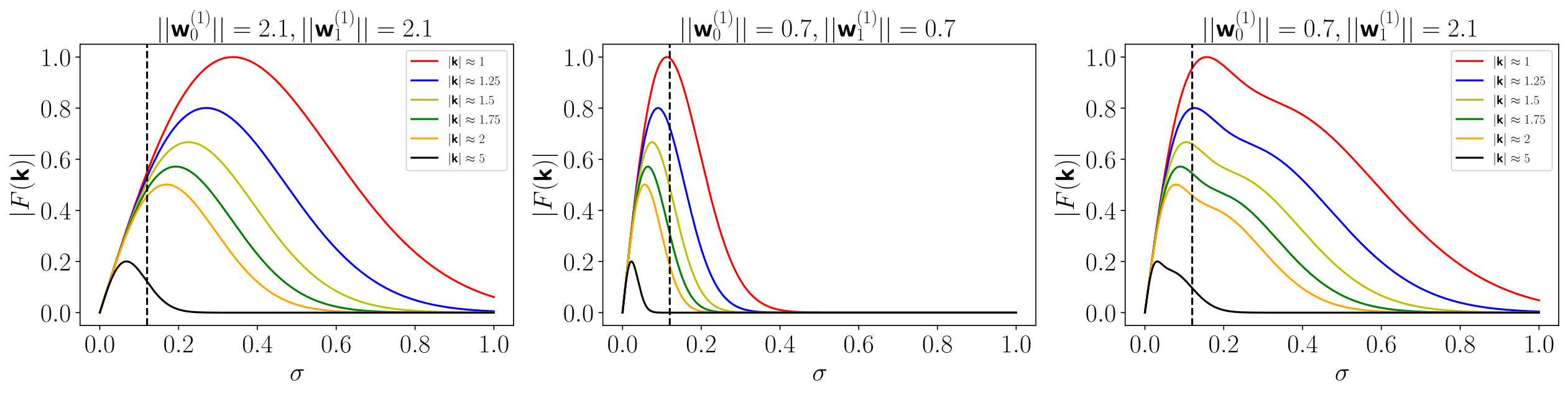}
    \vspace{-1em}
    \caption{\textbf{Toy example (a Gaussian-MLP $\mathbb{R}^2 \to \mathbb{R}$ with a single hidden layer consisting of two neurons): } The behavior of the spectrum against $\sigma$ and the weights of the network is shown. $\textbf{w}^{(1)}_0$ and $\textbf{w}^{(1)}_1$ are the first and second rows of the first-layer affine weight matrix. \textit{Left: } By decreasing $\sigma$, the network can include higher frequencies to the spectrum. However, as the high-frequency components are added to the spectrum, the relative energies of the low-frequency components decrease. \textit{Middle: } The network can still gain high energies for lower frequencies at lower $\sigma$ by decreasing $|\textbf{w}^{(1)}_0|$ and $|\textbf{w}^{(1)}_1|$. However, in this case, the high-frequency components are removed from the spectrum. \textit{Right: } By appropriately tuning $|\textbf{w}^{(1)}_0|$ and $|\textbf{w}^{(1)}_1|$, the spectrum can include  higher-frequency components while preserving the low-frequency energies. All the intensities are normalized for better comprehension.}
    \label{fig:toyex}
\end{figure*}

\section{Regularizing coordinate-MLPs}
\label{sec:regularizing}
 
The spectrum of a function is inherently related to the magnitude of its derivatives. For instance, consider a function $f:\mathbb{R} \to \mathbb{R}$ defined on a finite interval $\epsilon$. Then,

\[
   f(x) =  \int_{\infty}^{\infty} \hat{f}(k)e^{2\pi ikx}dk
\]
It follows that,

\begin{align}
    |\frac{df(x)}{dx}| & =      |2\pi i \int_{\infty}^{\infty} k\hat{f}(k)e^{2\pi ikx}dk|\\
    & \leq |2 \pi| \int_{\infty}^{\infty} |k\hat{f}(k)|dk.
\end{align}

Therefore,


\begin{equation}
\label{eq:dervsft}
    \max_{x \in \epsilon} |\frac{df(x)}{dx}| \leq |2 \pi| \int_{\infty}^{\infty} |k\hat{f}(k)|dk,
\end{equation}
 This is a tight lower-bound in the sense that the equality holds at $x = 0$. For multi-dimensions, we can encourage the spectrum to have a higher or lower frequency support by appropriately constraining the fluctuations along the corresponding directions. We shall now discuss how this fact may be utilized in regularizing coordinate-MLPs.

Let us consider a coordinate-MLP $f:\mathbb{R}^n \to \mathbb{R}$, which we factorise as $f = \tilde{f}\circ g$, where $\tilde{f}$ is the final layer. Then by the chain rule:

\begin{equation}
  \frac{ df}{d \vb{x}} = \frac{ \partial \tilde{f}}{ \partial \vb{y}} \cdot \frac{ \partial g}{ \partial \vb{x}}
\end{equation}

giving

\begin{equation}
\begin{split}
  \Bigl\lvert \frac{df}{d\vb{x}} \Bigr\rvert & = \sqrt{\Big[\frac{ \partial \tilde{f}}{ \partial \vb{y}} \cdot \frac{ \partial g}{ \partial \vb{x}}\Big] \Big[\frac{ \partial \tilde{f}}{ \partial \vb{y}} \cdot \frac{ \partial g}{ \partial \vb{x}}\Big]^T}\\
  & = \sqrt{\frac{ \partial \tilde{f}}{ \partial \vb{y}} \cdot \mathbf{J} \mathbf{J}^T \cdot \Big[ \frac{ \partial \tilde{f}}{ \partial \vb{y}}\Big]^T },
\end{split}
\end{equation}

where $\mathbf{J}$ is the Jacobian of $g$. Let $\mathbf{A} = \mathbf{J} \mathbf{J}^T$.
In order to minimize $\lvert{\frac{df}{d\vb{x}}}\rvert$, it suffices to minimize the magnitude of the components of $\mathbf{J}$, which is equivalent to minimizing the trace of $\mathbf{A}$. Recall that $tr(A) = \sum_k \lambda_k$, where $\lambda_k$ are the eigenvalues of $\textbf{A}$. One has

\begin{equation}
    \lim_{\lvert{\epsilon}\rvert \to 0} \frac{\lvert{g(\vb{x}) -g(\vb{x} + \epsilon u_k)}\rvert }{\lvert{\epsilon u_k}\rvert} = \sqrt{\lambda_k},
\end{equation}

where $u_k$ are the eigenvectors of $\textbf{A}$.  Therefore, minimizing an eigenvalue of $\textbf{A}$ is equivalent to restricting the fluctuations of $g$ along the direction of its associated eigenvector. By this logic, the ideal regularization procedure would be to identify the directions in the spectrum of the target signal where the frequency support is low, and restrict $\lambda_k$'s corresponding to those directions. According to our derivations in the previous sections, this would only be possible by regularizing   $\vb{w}_i$'s, since in practice, we use hyperparameters for coordinate-MLPs that allow higher bandwidth. However, this is a cumbersome task and we empirically found that there is a much more simpler approximation for this procedure that can give equivalent results as,


\begin{equation}
    \mathcal{L}_r =  \frac{\lvert{g(\bar{\vb{x}}) -g(\bar{\vb{x}} + \vb{\xi}  )\rvert}}{\lvert{\vb{\xi}\rvert}},
\end{equation}

where $\bar{\vb{x}}$ are randomly sampled from the coordinate space and $\vb{\xi} \sim \mathcal{N}(0, \Sigma)$ where $\Sigma$ is diagonal with small values. The total loss function for the coordinate-MLP then becomes $\mathcal{L}_{total} = \mathcal{L}_{MSE} + \varepsilon \mathcal{L}_{r}$ where $\varepsilon$ is a small scalar coefficient and $\mathcal{L}_{MSE}$ is the ususal mean squared error loss. The total loss can also be interpreted as encouraging the networks to obtain the smoothest possible solution while perfectly fitting the training samples. By the same argument, one can apply the above regularization on an arbitrary layer, including the output, although we empirically observed that the penultimate layer performs  best.





\section{Experiments}
\label{sec:experiments}

In this section, we will show that the insights developed thus far extend well to deep networks in practice. 

\subsection{Encoding signals with uneven sampling}

The earlier sections showed that coordinate-MLPs tend to suppress low frequencies when the network attempts to add higher frequencies to the spectrum. This can lead to poor performance when the target signal is unevenly sampled because, if the sampling is dense, the network needs to incorporate higher frequencies to properly encode the signal. Thus, at regions where the sampling is sparse, the network fails to produce smooth interpolations, resulting in noisy reconstructions. On the other hand, if the bandwidth of the network is restricted via hyperparameters or depth, the network can smoothly interpolate in sparse regions, but fails to encode information with high fidelity  at dense regions. Explicit regularization can aid the network in finding a properly balanced spectrum in the solution space. Fig.~\ref{fig:1dreg} shows an example for encoding a 1D signal. Fig.~\ref{fig:uneven} illustrates a qualitative example in encoding a 2D image. Table \ref{tab:uneven} depicts quantitative results on the natural dataset by \citet{tancik2020fourier}.

\begin{figure*}[!htp]
    \centering
    \includegraphics [width = 2.\columnwidth]{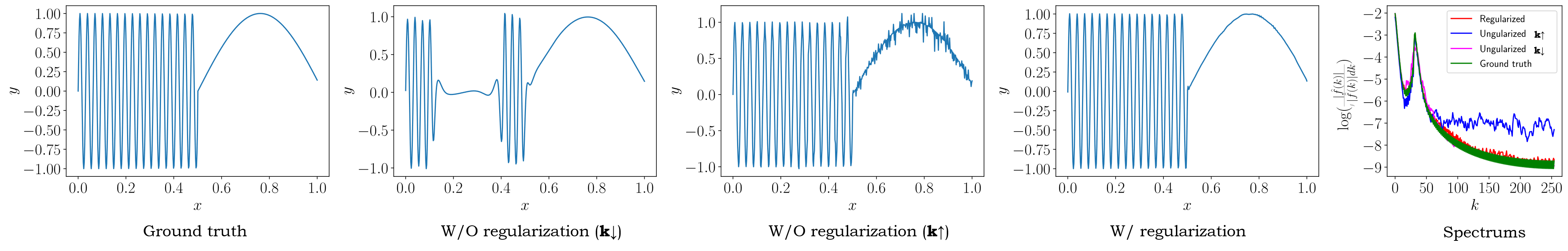}
    \vspace{-1em}
    \caption{\textbf{The effect of explicit regularization.} When trained under conventional means, coordinate MLPs cannot generalize well at both higher and lower ends of the spectrum. This hinders the generalization performance of coordinate-MLPs when the target signal comprises regions with different spectral properties. When the network has insufficient bandwidth, the network cannot correctly capture high-frequency modes. When the network is tuned to have a higher bandwidth, the network fails at modeling lower frequencies. In contrast, the network can preserve a better spectrum with the proposed regularization scheme. This example uses a $4$-layer sinusoid-MLP trained with $33\%$ of the total samples.}
    \label{fig:1dreg}
\end{figure*}

\begin{figure}[!tp]
    \centering
    \includegraphics[width=1.0\columnwidth]{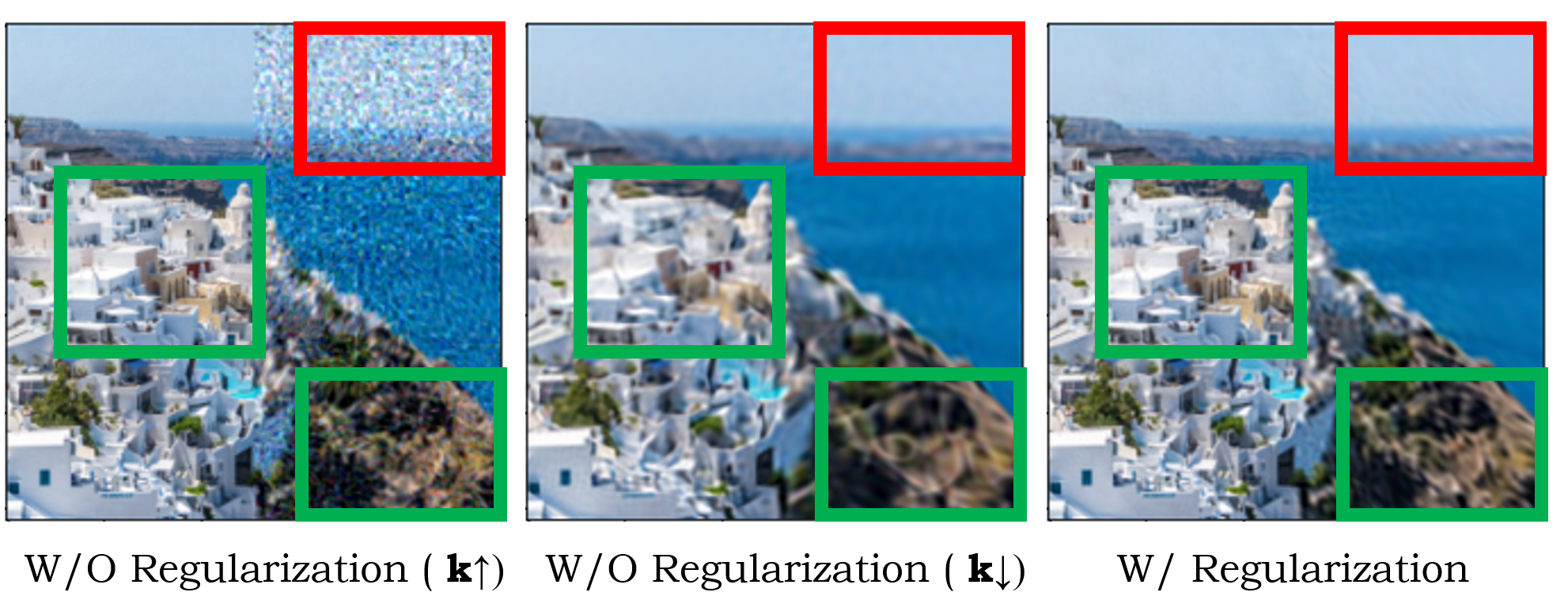}
    \vspace{-2em}
    \caption{ \textbf{Qualitative results for encoding signals with uneven sampling (zoom in for a better view)}. A Gaussian-MLP is trained to encode an image, where the left half of the image is sampled densly, and the right half is sampled with $10\%$ pixels. The reconstruction results are shown. When the bandwidth of the Gaussian-MLP is adjusted to match the sampling procedure of a particular half, the other half demonstrates poor reconstruction. In contrast, when regularized, coordinate-MLPs can preserve both low and high frequencies, giving a balanced reconstruction. Contrast the green and red areas across each setting.}
    \label{fig:uneven}
\end{figure} 

\subsection{Encoding signals with different local spectral properties}

The difficulty in generalizing well across different regions in the spectrum hinders encoding natural signals such as images when the sampling is sparse, as they tend to contain both ``flat" and ``fluctuating" regions, and the network has no prior on how to interpolate in these different regions. See Fig.~\ref{fig:regions} for an example. As evident, the coordinate-MLP struggles to smoothly interpolate between samples under a single hyperparameter setting. Table.~\ref{tab:regions} shows quantitative results over a subset of the STL dataset. As reported, the proposed regularization scheme is able to produce better results.

\begin{figure}[!htp]
    \centering
    \includegraphics[width=0.9\columnwidth]{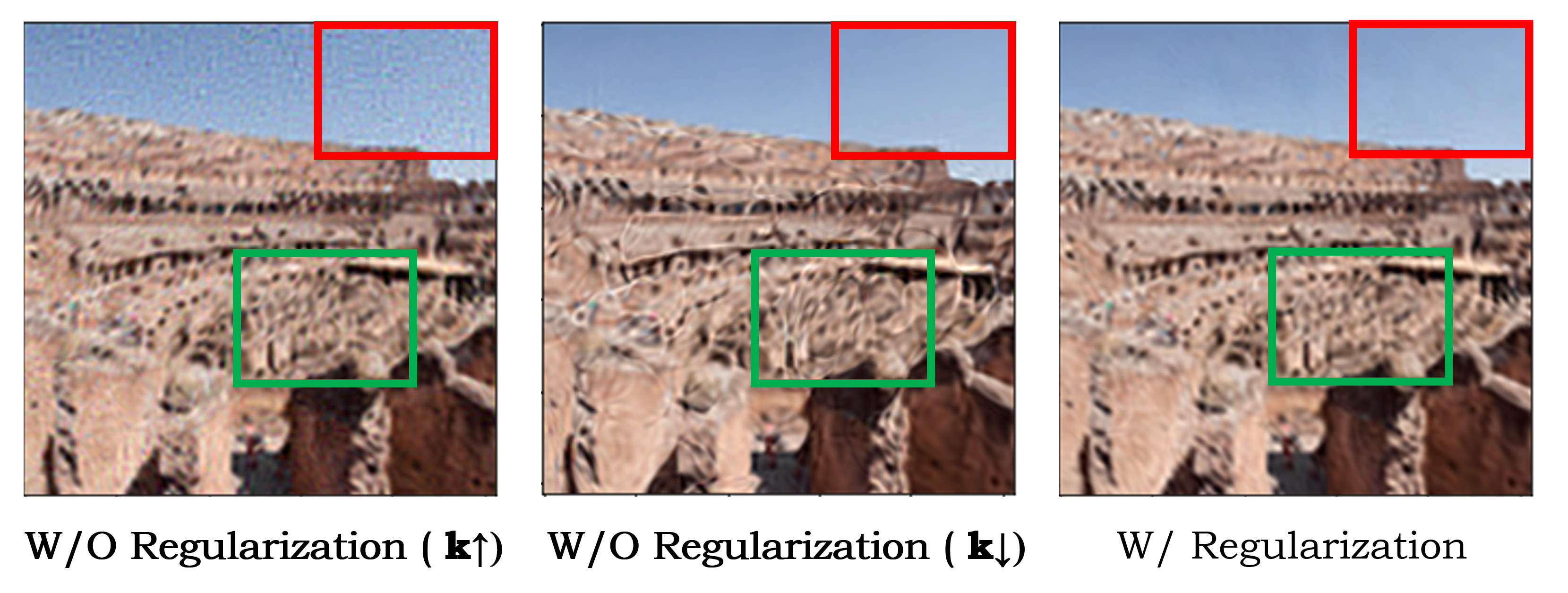}
    \vspace{-1em}
    \caption{ \textbf{A single hyperparameter setting cannot generalize well across the spectrum unless regularized (better viewed in zoom).} With sparse sampling, coordinate-MLPs exhibit inferior reconstruction performance across regions with different spectra. In comparison, regularized networks can achieve the best of both worlds. Compare the highlighted areas across each setting.}
    \label{fig:regions}
\end{figure} 

\subsection{Effect of increasing capacity}

Our derivations in Sec.~\ref{sec:FT} and \ref{sec:depth} showed that shallow coordinate-MLPs tend to suppress lower frequencies when the capacity of the network is increased via depth or hyperparameters. Our empirical results show that this is indeed the case for deeper networks as well (see Fig.~\ref{fig:depthvsspectrum}).

\begin{figure}[!htp]
    \centering
    \includegraphics[width=1.\columnwidth]{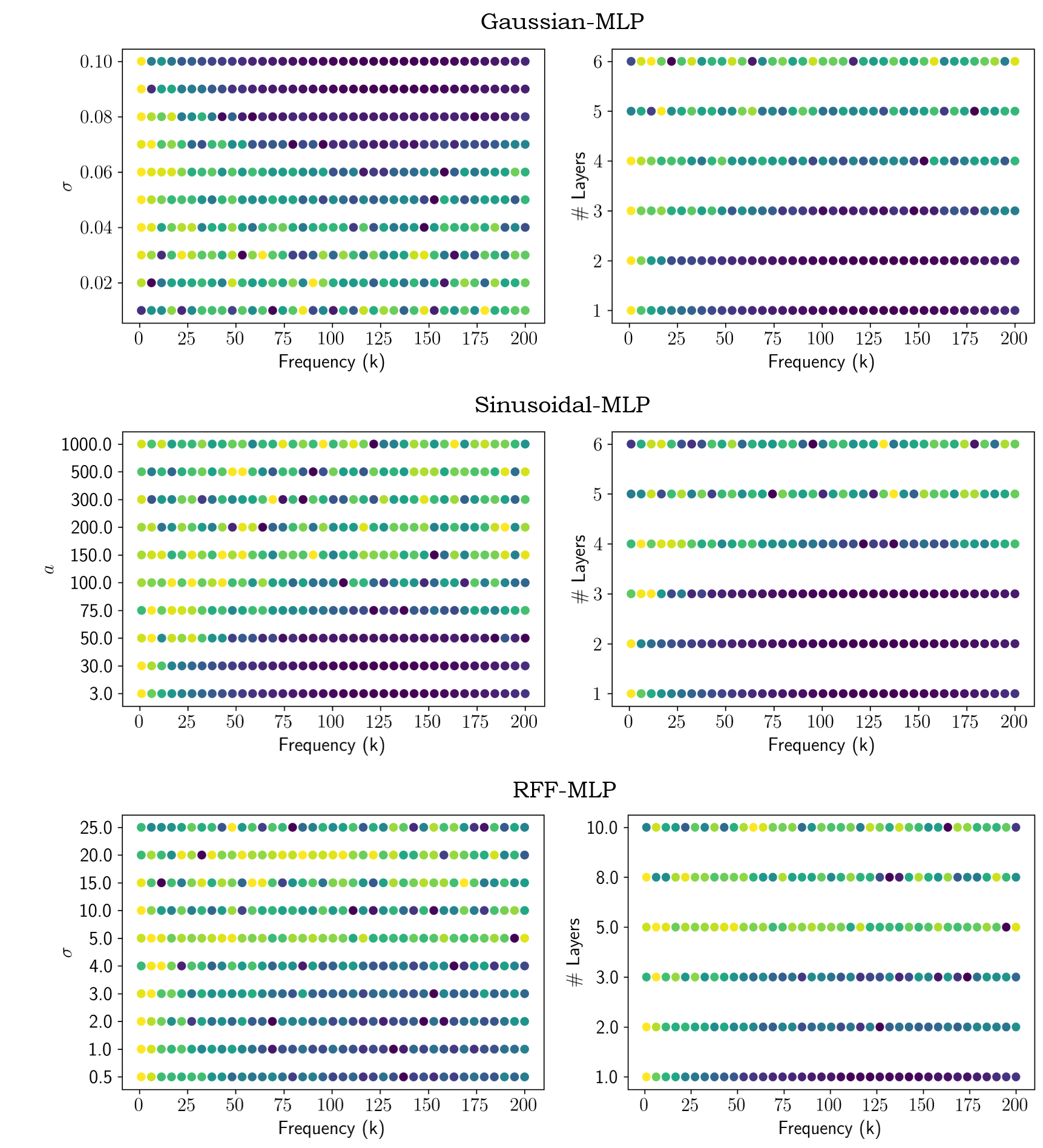}
    \vspace{-1em}
    \caption{ \textbf{Spectra of the coordinate-MLPs against hyperparameters and depth.} We train each network instance to encode a sound wave with $33\%$ sampling. The heat indicates the intensity of the corresponding frequency component after trained with the MSE loss. As illustrated, when the capacity of the network is increased via hyperparameters or depth, the networks tend to converge to solutions with suppressed low-frequency components.}
    \label{fig:depthvsspectrum}
\end{figure}

\begin{table}[!htp]
\centering
       \begin{tabular}{||c|c|c|c||}

\multicolumn{4}{c}{W/O Regularization ($\vb{k}\downarrow$)}\\
\hline
Type & L-PSNR & R-PSNR & T-PSNR  \\
\hline
Gaussian-MLP & $30.12$ & $22.47$ & $27.72$  \\
Sinusoid-MLP &$30.49$ & $21.93$ & $26.89$  \\
RFF-MLP &  $29.89$ & $22.33$ & $26.44$ \\
\hline
\multicolumn{4}{c}{W/O Regularization ($\vb{k}\uparrow$)}\\
\hline
Type & L-PSNR & R-PSNR & T-PSNR  \\
\hline
Gaussian-MLP & $33.43$ & $18.14$ & $22.80$   \\
Sinusoid-MLP & $32.17$ & $19.30$ &$23.49$ \\
RFF-MLP & $32.24$ & $19.11$ & $22.19$\\
\hline
\multicolumn{4}{c}{Regularized}\\
\hline
Type & L-PSNR & R-PSNR & T-PSNR  \\
\hline
Gaussian-MLP (Reg) & $33.11$ & $23.31$ &  $30.15$  \\
Sinusoid-MLP (Reg) & $31.59$ & $22.66$ & $29.94$  \\
RFF-MLP (Reg) &$31.99$ &$22.91$  & $29.59$ \\
\hline
\end{tabular}
 \vspace{-0.7em}
  \caption{\textbf{Encoding images with uneven sampling.} In each training instance, the left half of the image is sampled densely, and the right half is sampled with $10\%$ pixels. With unregularized coordinate-MLPs, when the hyperparameters are tuned to match the sampling of a particular half, the reconstruction of the other half is poor. In contrast, the encoding performance is balanced when regularized. We use $4$-layer networks for this experiment.}
\label{tab:uneven}
\end{table}

\begin{table}[!htp]

\centering
       \begin{tabular}{||c|c|c||}

\hline
Type & W/O Regularization &  Regularized  \\
\hline
Gaussian-MLP & $21.91$ & $24.67$\\
Sinusoid-MLP &$21.88$ & $24.55$   \\
RFF-MLP &  $21.16$ & $23.94$ \\
\hline
\end{tabular}
 \vspace{-0.7em}
  \caption{\textbf{Encoding images with sparse sampling.} We compare the encoding performance over a subset ($30$) images of the STL dataset \cite{coates2011analysis} with $10\%$ sampling. Regularized coordinate-MLPs show superior performance due to better interpolation properties. We use $4$-layer coordinate-MLPs for this experiment.}
\label{tab:regions}
\end{table}

\section{Conclusion}
\label{sec:conclusion}

We show that the traditional implicit regularization assumptions do not hold in the context of coordinate-MLPs. We focus on establishing plausible reasoning for this phenomenon from a Fourier angle and discover that coordinate-MLPs tend to suppress lower frequencies when the capacity is increased unless explicitly regularized. We further show that the developed insights are valid in practice.

\bibliography{example_paper}
\bibliographystyle{icml2022}

\end{document}